\documentclass[11pt]{article}

\usepackage{graphicx} 

\usepackage[square]{natbib}

\usepackage{wrapfig}
\usepackage{caption}
\usepackage{booktabs}
\usepackage{amsmath}

\usepackage{hyperref}

\usepackage{fullpage}

\newcommand{\E}{\mathbb{E}}

\usepackage{color}
\newcommand{\matt}[1]{}
\newcommand{\sj}[1]{}

\newcommand{\Hs}{\mathcal H}

\newcommand{\Us}{\mathcal U}

\newcommand{\dMMD}{d_\mathrm{MMD}}

\newcommand{\distp}{\mathbb P}
\newcommand{\distq}{\mathbb Q}

\newcommand{\fdivs}{$\phi$-divergences}
\newcommand{\fdiv}{$\phi$-divergence}

\newcommand{\reals}{\mathbb R}

\usepackage{amsmath,amsfonts}

\DeclareMathOperator{\trace}{\mathrm{tr}}
\DeclareMathOperator{\diag}{\mathrm{diag}}

\DeclareMathOperator{\Var}{\mathrm{Var}}

\newcommand{\argmax}{\operatornamewithlimits{argmax}}

\usepackage{amsthm}
\newtheorem{theorem}{Theorem}[section]
\newtheorem{lemma}{Lemma}[section]
\newtheorem{corollary}{Corollary}[section]

\newtheorem{proposition}{Proposition}[section]
\newtheorem{principle}{Principle}[section]
\theoremstyle{definition}
\newtheorem{definition}{Definition}[section]

\newtheorem{fact}{Fact}[section]

\title{Distributionally Robust Optimization and \\ Generalization in Kernel Methods}

\author{Matthew Staib \\ MIT CSAIL \\ mstaib@mit.edu \and 
Stefanie Jegelka \\ MIT CSAIL \\ stefje@csail.mit.edu}

\date{}

\begin{document}

\maketitle

\begin{abstract} 
Distributionally robust optimization (DRO) has attracted attention in machine learning due to its connections to regularization, generalization, and robustness. 
Existing work has considered uncertainty sets based on \fdivs\ and Wasserstein distances, each of which have drawbacks.
In this paper, we study DRO with uncertainty sets measured via maximum mean discrepancy (MMD). 
We show that MMD DRO is roughly equivalent to regularization by the Hilbert norm and, 
as a byproduct, reveal deep connections to classic results in statistical learning. 
In particular, we obtain an alternative proof of a generalization bound for Gaussian kernel ridge regression via a DRO lense. 
The proof also suggests a new regularizer.
Our results apply beyond kernel methods: we derive a generically applicable approximation of MMD DRO, and show that it generalizes recent work on variance-based regularization.
\end{abstract} 


\section{Introduction}
Distributionally robust optimization (DRO) 
is an attractive tool for improving machine learning models.
Instead of choosing a model $f$ to minimize empirical risk $\E_{x \sim \hat \distp_n}[\ell_f(x)] = \frac1n \sum_i \ell_f(x_i)$, an adversary is allowed to perturb the sample distribution within a set $\Us$ centered around the empirical distribution $\hat \distp_n$. DRO seeks a model that performs well regardless of the perturbation: $\inf_f \sup_{\distq \in \Us} \E_{x \sim \distq}[\ell_f(x)]$. The induced robustness can directly imply generalization: if the data that forms $\hat \distp_n$ is drawn from a population distribution $\distp$, and $\Us$ is large enough to contain $\distp$, then we implicitly optimize for $\distp$ too and the DRO objective value 
upper bounds out of sample performance.
More broadly, robustness has gained attention due to adversarial examples \citep{goodfellow2015explaining,szegedy2014intriguing,madry2018towards}; indeed, DRO generalizes robustness to adversarial examples \citep{sinha2018certifying,staib2017distributionally}.

In machine learning, the DRO uncertainty set $\Us$ has so far always been defined as a \fdiv\ ball or Wasserstein ball around the empirical distribution $\hat \distp_n$. 
These choices are convenient, due to a number of structural results. For example, DRO with $\chi^2$-divergence is roughly equivalent to regularizing by variance~\citep{gotoh2015robust,lam2016robust,Namkoong2017Variance}, and the worst case distribution $\distq \in \Us$ can be computed exactly in $O(n\log n)$~\citep{staib2019distributionally}. Moreover, DRO with Wasserstein distance is asymptotically equivalent to certain common norm penalties~\citep{gao2017wasserstein}, and the worst case $\distq \in \Us$ can be computed approximately in several cases~\citep{MohajerinEsfahani2018,gao2016distributionally}.
These structural results are key, because the most challenging part of DRO is solving (or bounding) the DRO objective.

However, there are substantial drawbacks to these two types of uncertainty sets. 
Any \fdiv\ uncertainty set $\Us$ around $\hat \distp_n$ contains only distributions with the same (finite) support as $\hat \distp_n$. Hence, the population $\distp$ is typically \emph{not} in $\Us$, and so the DRO objective value cannot directly certify out of sample performance. 
Wasserstein uncertainty sets do not suffer from this problem. But, they are more 
computationally expensive, and the above key results on equivalences and computation 
need nontrivial assumptions on the loss function and the specific ground distance metric used.

In this paper, we introduce and develop a new class of DRO problems, where the uncertainty set $\Us$ is defined with respect to the maximum mean discrepancy (MMD)~\citep{Gretton:2012:KTT:2188385.2188410}, a kernel-based distance between distributions.
MMD DRO complements existing approaches and avoids some of their drawbacks, e.g., unlike \fdivs, the uncertainty set $\Us$ will contain $\distp$ if the radius is large enough.

First, we show that MMD DRO is roughly equivalent to regularizing by the Hilbert norm $\lVert \ell_f \rVert_\Hs$ of the loss $\ell_f$ (not the model $f$).
While, in general, $\lVert \ell_f \rVert_\Hs$ may be difficult to compute, we show settings in which it is tractable. 
Specifically, for kernel ridge regression with a Gaussian kernel, we prove a bound on $\lVert \ell_f \rVert_\Hs$ that, as a byproduct, yields generalization bounds that match (up to a small constant) the standard ones.
Second, beyond kernel methods, we show how MMD DRO generalizes variance-based regularization.
%
Finally, we show how MMD DRO can be efficiently approximated empirically, and in fact generalizes variance-based regularization.

Overall, our results offer deeper insights into the landscape of regularization and robustness approaches, and a more complete picture of the effects of different divergences for defining robustness. In short, our contributions are:
\vspace{-5pt}
\begin{enumerate}\setlength{\itemsep}{0pt}
	\item We prove fundamental structural results for MMD DRO, and its rough equivalence to penalizing by the Hilbert norm of the loss.
	\item We give a new generalization proof for Gaussian kernel ridge regression by way of DRO. 
	Along the way, we prove bounds on the Hilbert norm of products of functions that may be of independent interest.
	\item Our generalization proof suggests a new regularizer for Gaussian kernel ridge regression.
	\item We derive a computationally tractable approximation of MMD DRO, with application to general learning problems, and we show how the aforementioned approximation generalizes variance regularization.
\end{enumerate}


\section{Background and Related Work}
Distributionally robust optimization (DRO)~\citep{doi:10.1287/opre.1090.0795,Bertsimas2018}, introduced by \citet{scarf1958min}, asks to not only perform well on a fixed problem instance (parameterized by a distribution), but simultaneously for a range of problems, each determined by a distribution in an \emph{uncertainty set} $\Us$. This results in more robust solutions. The uncertainty set plays a key role: it implicitly defines the induced notion of robustness.
%
%
%
The DRO problem we address asks to learn a model $f$ that solves
\begin{equation}
	\mathrm{(DRO)} \quad\qquad\qquad \inf_f\, \sup_{\distq \in \Us}\, \E_{x\sim\distq}[\ell_f(x)],
	\label{eq:dro-problem}
\end{equation}
where $\ell_f(x)$ is the loss incurred under prediction $f(x)$. 

In this work, we focus on \emph{data-driven DRO}, where $\Us$ is centered around an empirical sample $\hat \distp_n  = \frac1n \sum_{i=1}^n \delta_{x_i}$, and its size is determined in a data-dependent way.
Data-driven DRO 
yields a natural approach for certifying out-of-sample performance.
\begin{principle}[DRO Generalization Principle]
\label{principle:dro-generalization}
Fix any model $f$.
Let $\Us$ be a set of distributions containing $\hat \distp_n$. Suppose $\Us$ is large enough so that, with probability $1-\delta$, $\Us$ contains the population $\distp$.
Then with probability $1-\delta$, the population loss $\E_{x\sim\distp}[\ell_f(x)]$ is bounded by
\begin{align}
	\E_{x\sim\distp}[\ell_f(x)]\; \leq\; \sup_{\distq \in \Us} \E_{x\sim\distq}[\ell_f(x)].
\end{align}
\end{principle}

Essentially, if the uncertainty set $\Us$ is chosen appropriately, the corresponding DRO problem gives a high probability bound on population performance.
The two key steps in using Principle~\ref{principle:dro-generalization} are \textbf{1.}~arguing that $\Us$ actually contains $\distp$ with high probability (e.g. via concentration); \textbf{2.}~solving the DRO problem on the right hand side, or an upper bound thereof.

In practice, $\Us$ is typically chosen as a ball of radius $\epsilon$ around the empirical sample $\hat \distp_n$: $\Us = \{\distq : d(\distq, \hat \distp_n) \leq \epsilon\}$.
Here, $d$ is a discrepancy between distributions, and is of utmost significance: the choice of $d$ determines how large $\epsilon$ must be, and how tractable the DRO problem is.

In machine learning, two choices of the divergence $d$ are prevalent, \fdivs\ \citep{doi:10.1287/mnsc.1120.1641,duchi2016statistics,lam2016robust}, and Wasserstein distance~\citep{MohajerinEsfahani2018,NIPS2015_5745,blanchet2016robust}. The first option, \fdivs, have the form $d_\phi(\distp,\distq) = \int \phi(d\distp / d\distq) \, d\distq$. In particular, they include the $\chi^2$-divergence, which makes the DRO problem equivalent to regularizing by variance~\citep{gotoh2015robust,lam2016robust,Namkoong2017Variance}.
Beyond better generalization, variance regularization has applications in fairness~\citep{pmlr-v80-hashimoto18a}.
%
However, a major shortcoming of DRO with \fdivs\ is that the ball $\Us = \{\distq : d_{\phi}(\distq, \distp_0) \leq \epsilon\}$ only contains distributions $\distq$
whose support is contained in the support of $\distp_0$. If $\distp_0 = \hat \distp_n$ is an empirical distribution on $n$ points, the ball $\Us$ only contains distributions with the same finite support. Hence, the population distribution $\distp$ typically cannot belong to $\Us$, and it is not possible to certify out-of-sample perfomance by Principle~\ref{principle:dro-generalization}.


The second option, Wasserstein distance, is based on a distance metric $g$ on the data space. 
%
The $p$-Wasserstein distance $W_p$ between measure $\mu,\nu$ is given by $W_p(\mu,\nu) = \inf\{ \int g(x,y)^p \, d\gamma(x,y) : \gamma \in \Pi(\mu, \nu)\}^{1/p}$, where $\Pi(\mu, \nu)$ is the set of couplings of $\mu$ and $\nu$~\citep{villani2008optimal}.
Wasserstein DRO has a key benefit over \fdivs: the set $\Us = \{\distq : W_p(\distq, \distp_0) \leq \epsilon \}$ contains continuous distributions. However, Wasserstein distance is much harder to work with, and nontrivial assumptions are needed to derive the necessary structural and algorithmic results for solving the associated DRO problem.
Further, to the best of our knowledge, in all Wasserstein DRO work so far, the ground metric $g$ is limited to slight variations of either a Euclidean or Mahalanobis metric~\citep{blanchet2017data,blanchet2018optimal}.
Such metrics may be a poor fit for complex data such as images or distributions.
Concentration results stating $W_p(\distp, \hat \distp_n)$ with high probability also typically require a Euclidean metric, e.g.~\citep{Fournier2015}. 
These assumptions restrict the extent to which Wasserstein DRO can utilize complex, nonlinear structure in the data.

\paragraph{Maximum Mean Discrepancy (MMD).}
MMD is a distance metric between distributions that leverages kernel embeddings.
Let $\Hs$ be a reproducing kernel Hilbert space (RKHS) with kernel $k$ and norm $\lVert \cdot \rVert_\Hs$. MMD is defined as follows:
\begin{definition}
\label{definition:mmd}
The \emph{maximum mean discrepancy (MMD)} between distributions $\distp$ and $\distq$ is
\begin{align}
  \dMMD(\distp, \distq) := \sup_{f \in \Hs : \lVert f \rVert_\Hs \leq 1} \E_{x \sim \distp}[f(x)] - \E_{x \sim \distq}[f(x)].
\end{align}
\end{definition}

\begin{fact}
\label{fact:mean-embedding}
Define the mean embedding $\mu_\distp$ of the distribution $\distp$ by $\mu_\distp = \E_{x \sim \distp}[k(x, \cdot)]$.
Then the MMD between distributions $\distp$ and $\distq$ can be equivalently written
\begin{align}
  \dMMD(\distp, \distq) = \lVert \mu_\distp - \mu_\distq \rVert_\Hs.
\end{align}
\end{fact}
MMD and (more generally) kernel mean embeddings have been used in many applications, particularly in two- and one-sample tests~\citep{Gretton:2012:KTT:2188385.2188410,NIPS2017_6630,pmlr-v48-liub16,pmlr-v48-chwialkowski16} and in generative modeling~\citep{Dziugaite:2015:TGN:3020847.3020875,pmlr-v37-li15,sutherland2017generative,binkowski2018demystifying}. 
We refer the interested reader to the monograph by~\citet{MAL-060}.
MMD admits efficient estimation, as well as fast convergence properties, which are of chief importance in our work.

\paragraph{Further related work.}
Beyond \fdivs\ and Wasserstein distances, work in operations research has considered DRO problems that capture uncertainty in moments of the distribution, e.g.~\citep{doi:10.1287/opre.1090.0741}. 
These approaches typically focus on first- and second-order moments; in contrast, an MMD uncertainty set allows high order moments to vary, depending on the choice of kernel.




Robust and adversarial machine learning have strong connections to our work and DRO more generally.
Robustness to adversarial examples~\citep{szegedy2014intriguing,goodfellow2015explaining}, where individual inputs to the model are perturbed in a small ball, can be cast as a robust optimization problem~\citep{madry2018towards}.
When the ball is a norm ball, this robust formulation is a special case of Wasserstein DRO~\citep{sinha2018certifying,staib2017distributionally}. 
\citet{xu2009robustness} study the connection between robustness and regularization in SVMs, and perturbations within a (possibly Hilbert) norm ball.
Unlike our work, their results are limited to SVMs instead of general loss minimization. 
Moreover, they consider only perturbation of individual data points instead of shifts in the entire \emph{distribution}. 
\citet{bietti2019kernel} show that many regularizers used for neural networks can also be interpreted in light of an appropriately chosen Hilbert norm~\citep{bietti2019group}.


\section{Generalization bounds via MMD DRO}
\label{section:implications-generalization}

The main focus of this paper is Distributionally Robust Optimization where the uncertainty set is defined via the MMD distance $\dMMD$:
%
\begin{equation}
  \inf_f\;\;\; \sup_{\distq : \dMMD(\distq, \hat \distp_n) \leq \epsilon}\;\;\; \E_{x \sim \distq}[\ell_f(x)].
  \label{eq:mmd-dro-prob}
\end{equation}

One motivation for considering MMD in this setting are its possible implications for Generalization. Recall that for the DRO Generalization Principle~\ref{principle:dro-generalization} to apply, the uncertainty set $\mathcal{U}$ must contain the population distribution with high probability. To ensure this, the radius of $\mathcal{U}$ must be large enough. But, the larger the radius, the more pessimistic is the DRO minimax problem, which may lead to over-regularization. This radius depends on how quickly $\dMMD(\distp, \hat \distp_n)$ shrinks to zero, i.e., on the empirical accuracy of the divergence.

In contrast to Wasserstein distance, which converges at a rate of $O(n^{-1/d})$~\citep{Fournier2015}, 
MMD between the empirical sample $\hat \distp_n$ and population $\distp$ shrinks as $O(n^{-1/2})$:
\begin{lemma}[Modified from \citep{MAL-060}, Theorem 3.4]
\label{lemma:mmd-convergence}
Suppose that $k(x,x) \leq M$ for all $x$. Let $\hat \distp_n$ be an $n$ sample empirical approximation to $\distp$. Then with probability $1-\delta$,
\begin{align}
  \dMMD(\distp, \hat \distp_n) \leq 2 \sqrt{\frac{M}{n}} + \sqrt{\frac{2 \log(1/\delta)}{n}}.
\end{align}
\end{lemma}
With Lemma~\ref{lemma:mmd-convergence} in hand, we conclude a simple high probability bound on out-of-sample performance:
\begin{corollary}
\label{corollary:mmd-dro-principle}
Suppose that $k(x,x) \leq M$ for all $x$. 
Set $\epsilon = 2\sqrt{M/n} + \sqrt{2\log(1/\delta)/n}$. 
Then with probability $1-\delta$, we have the following bound on population risk:
\begin{align}
  \E_{x \sim \distp}[\ell_f(x)] \leq \sup_{\distq : \dMMD(\distq, \hat \distp_n) \leq \epsilon} \E_{x\sim \distq}[\ell_f(x)].
  \label{eq:mmd-dro-problem}
\end{align}
\end{corollary}
We refer to the right hand side as the DRO adversary's problem.
In the next section we develop results that enable us to bound its value, and consequently bound the DRO problem~\eqref{eq:mmd-dro-prob}.

\subsection{Bounding the DRO adversary's problem}


The DRO adversary's problem seeks the distribution $\distq$ in the MMD ball so that $\E_{x\sim \distq}[\ell_f(x)]$ is as high as possible.
Reasoning about the optimal worst-case $\distq$ is 
the main difficulty in DRO.
With MMD, we take two steps for simplification. First, 
instead of directly optimizing over distributions, we optimize over their mean embeddings in the Hilbert space (described in Fact~\ref{fact:mean-embedding}).
Second, while the adversary's problem~\eqref{eq:mmd-dro-problem} makes sense for general $\ell_f$, we assume that the loss $\ell_f$ is in $\Hs$.
%
In case $\ell_f \not\in\Hs$, 
often $k$ is a universal kernel, meaning under mild conditions $\ell_f$ can be approximated arbitrarily well by a member of $\Hs$~\citep[Definition 3.3]{MAL-060}.

With the additional assumption that  $\ell_f \in \Hs$, the risk $\E_{x \sim \distp}[\ell_f(x)]$
can also be written as $\langle \ell_f, \mu_\distp \rangle_\Hs$. Then we obtain
\begin{align}
  \sup_{\distq : \dMMD(\distq, \distp) \leq \epsilon} \E_{x \sim \distq}[\ell_f(x)] \leq \sup_{\mu_\distq \in \Hs : \lVert \mu_\distq - \mu_\distp \rVert_\Hs \leq \epsilon} \langle \ell_f, \mu_\distq \rangle_\Hs,
  \label{eq:bound-mmd-dro-by-mean-embedding}
\end{align}
where we have an inequality because not every function in $\Hs$ is the mean embedding of some probability distribution. 
If $k$ is a characteristic kernel~\citep[Definition 3.2]{MAL-060}, the mapping $\distp \mapsto \mu_\distp$ is injective. In this case, the only looseness in the bound is due to discarding the constraints that $\distq$ integrates to one and is nonnegative.
However it is difficult to constraint the mean embedding $\mu_\distq$ in this way as it is a function. 

The mean embedding form of the problem is simpler to work with, and leads to further interpretations.
\begin{theorem}
\label{thm:sup-of-mean-embeddings-norm}
Let $\ell_f, \mu_\distp \in \Hs$. We have the following equality:
\begin{align}
  \sup_{\mu_\distq \in \Hs : \lVert \mu_\distq - \mu_\distp \rVert_\Hs \leq \epsilon} \langle \ell_f, \mu_\distq \rangle_\Hs = \langle \ell_f, \mu_\distp \rangle_\Hs + \epsilon \lVert \ell_f \rVert_\Hs = \E_{x \sim \distp}[\ell_f(x)] + \epsilon \lVert \ell_f \rVert_\Hs.
\end{align}
In particular, the optimal solution is $\mu_\distq^* = \mu_\distp + \frac{\epsilon}{\lVert \ell \rVert_\Hs} \ell_f$.
\end{theorem}
Combining Theorem~\ref{thm:sup-of-mean-embeddings-norm} with equation~\eqref{eq:bound-mmd-dro-by-mean-embedding} yields 
our main result for this section:
\begin{corollary}
\label{corollary:bound-dro-by-norm-penalty}
Let $\ell_f \in \Hs$, let $\distp$ be a probability distribution, and fix $\epsilon > 0$. Then,
\begin{align}
  \sup_{\distq : \dMMD(\distp, \distq) \leq \epsilon} \E_{x \sim \distq}[\ell_f(x)]\; &\leq\; \E_{x \sim \distp}[\ell_f(x)] + \epsilon \lVert \ell_f \rVert_\Hs \qquad \text{and therefore} \\
  \quad \inf_f \sup_{\distq : \dMMD(\distp, \distq) \leq \epsilon} \E_{x \sim \distq}[\ell_f(x)] \; &\leq\; \inf_f \E_{x \sim \distp}[\ell_f(x)] + \epsilon \lVert \ell_f \rVert_\Hs.
\end{align}
\end{corollary}
Combining Corollary~\ref{corollary:bound-dro-by-norm-penalty} with Corollary~\ref{corollary:mmd-dro-principle} shows that minimizing the empirical risk plus a norm on $\ell_f$ leads to a high probability bound on out-of-sample performance.
This result is similar to results that equate Wasserstein DRO to norm regularization.
For example, \citet{gao2017wasserstein} show that under appropriate assumptions on $\ell_f$, DRO with a $p$-Wasserstein ball is asymptotically equivalent to
$\E_{x \sim \hat \distp_n}[\ell_f(x)] + \epsilon \lVert \nabla_x \ell_f \rVert_{\hat\distp_n,q}$,
where $\lVert \nabla_x \ell_f \rVert_{\hat\distp_n,q} = \left(\frac1n \sum_{i=1}^n \lVert \nabla_x \ell_f(x_i) \rVert_*^{q}\right)^{1/q}$ measures a kind of $q$-norm average of $\lVert \nabla_x \ell_f(x_i) \rVert_*$ at each data point $x_i$ (here $q$ is such that $1/p + 1/q = 1$, and $\|\cdot\|_*$ is the dual norm of the metric defining the Wasserstein distance). \sj{Simplify this, I can't fully parse it.}


There are a few key differences between our result and that of~\citet{gao2017wasserstein}. First, the norms are different. 
Second, their result penalizes only the gradient of $\ell_f$, while ours penalizes $\ell_f$ directly.
Third, except for certain special cases, 
the Wasserstein results cannot serve as a true upper bound; there are higher order terms that only shrink to zero as $\epsilon \to 0$. 
Even worse, in high dimension $d$, the radius $\epsilon$ of the uncertainty set needed so that $\distp \in \mathcal{U}$ 
%
shrinks very slowly, as $O(n^{-1/d})$~\citep{Fournier2015}.



\section{Connections to kernel ridge regression}
\label{section:krr}

\matt{consider re-titling the section to be less specific on KRR, which we focus on in the subsection}
%
%
After applying Corollary~\ref{corollary:bound-dro-by-norm-penalty}, we are interested in solving:
\begin{align}
\label{eq:dro-l-composed-with-f}
  \inf_f  \E_{x \sim \hat \distp_n}[\ell_f(x)] + \epsilon \lVert \ell_f \rVert_\Hs.
\end{align}
Here, 
we penalize our model $f$ by $\lVert \ell_f \rVert_\Hs$. 
This looks 
similar to but is 
very different from the usual penalty $\lVert f \rVert_\Hs$ in kernel methods.
In fact, Hilbert norms of function compositions such as $\ell_f$ pose several challenges. For example, $f$ and $\ell_f$ may not belong to the same RKHS -- it is not hard to construct counterexamples, even when $\ell$ is merely quadratic. So, the objective \eqref{eq:dro-l-composed-with-f} is not yet computational.


Despite these challenges, we next develop tools that will allow us to bound $\lVert \ell_f \rVert_\Hs$ and use it as a regularizer. These tools may be of independent interest to bound RKHS norms of composite functions (e.g., for settings as in \citep{bietti2019kernel}).
Due to the difficulty of this task,
we specialize to Gaussian kernels $k_\sigma(x,y) = \exp(-\lVert x - y\rVert^2 / (2\sigma^2))$.
Since we will need to take care regarding the bandwidth $\sigma$, we explicitly write it out for the inner product $\langle \cdot, \cdot \rangle_\sigma$ and norm $\lVert \cdot \rVert_\sigma$, of the corresponding RKHS $H_\sigma$. 


To make the setting concrete,
consider kernel ridge regression, with Gaussian kernel $k_\sigma$.
As usual, we assume there is a simple target function $h$ that fits our data: $h(x_i) = y_i$.
Then the loss $\ell_f$ of $f$ is $\ell_f(x) = (f(x) - h(x))^2$, so we wish to solve
\begin{align}
  \inf_f  \E_{x \sim \hat \distp_n}[(f(x) - h(x))^2] + \epsilon \lVert (f - h)^2 \rVert_\sigma.
  \label{eq:gaussian-krr-wrong-sigma}
\end{align}
\matt{note: if we use $\sigma$ here, we need to change all the $\sigma/\sqrt2$ later to $\sigma$, and $\sigma$ to $\sqrt2 \sigma$}

\subsection{Bounding norms of products}
To bound $\lVert (f - h)^2 \rVert_\sigma$, it will suffice to bound RKHS norms of products.
The key result for this subsection is the following deceptively simple-looking bound:
\begin{theorem}
\label{theorem:gaussian-product-bound}
Let $f, g \in \Hs_\sigma$, that is, the RKHS corresponding to the Gaussian kernel $k_\sigma$ of bandwidth $\sigma$. Then, $\lVert f g \rVert_{\sigma / \sqrt2} \leq \lVert f \rVert_\sigma \lVert g \rVert_\sigma$. 
\end{theorem}
Indeed, there are already subtleties: if $f,g \in \Hs_\sigma$, then, to discuss the norm of the product $f g$, we need to decrease the bandwidth from $\sigma$ to $\sigma / \sqrt2$.

We prove Theorem~\ref{theorem:gaussian-product-bound} via two steps.
First, we represent the functions $f, g$, and $fg$ \emph{exactly} in terms of traces of certain matrices.
This step is highly dependent on the specific structure of the Gaussian kernel.
Then, we can apply standard trace inequalities.
Proofs of both results are given in Appendix~\ref{appendix:gaussian-kernel-bounds}.
\begin{proposition}
\label{prop:write-as-traces}
Let $f, g \in \Hs_\sigma$ have expansions $f = \sum_i a_i k_\sigma(x_i, \cdot)$ and $g = \sum_j b_j k_\sigma(x_j, \cdot)$.
For shorthand denote by $z_i = \phi_{\sqrt2 \sigma}(x_i)$ the (possibly infinite) feature expansion of $x_i$ in $\Hs_{\sqrt2 \sigma}$.
Then,
\begin{align*}
  \lVert f g \rVert_{\sigma / \sqrt2}^2 = \trace(A^2 B^2), 
  \quad
  \lVert f \rVert_\sigma^2 = \trace(A^2), 
  \quad \text{and} \quad \lVert g \rVert_\sigma^2 = \trace(B^2), 
\end{align*}
where $A = \sum_i a_i z_i z_i^T$ and $B = \sum_j a_j z_j z_j^T$.
\end{proposition}
\begin{lemma}
\label{lemma:trace-submultiplicative}
Let $X, Y$ be symmetric and positive semidefinite. Then $\trace(XY) \leq \trace(X)\trace(Y)$.
\end{lemma}

With these intermediate results in hand, we can prove the main bound of interest:
\begin{proof}[Proof of Theorem~\ref{theorem:gaussian-product-bound}]
By Proposition~\ref{prop:write-as-traces}, we may write
\begin{align*}
  \lVert f g \rVert_{\sigma / \sqrt2}^2 = \trace(A^2 B^2), 
  \quad
  \lVert f \rVert_\sigma^2 = \trace(A^2), 
  \quad \text{and} \quad \lVert g \rVert_\sigma^2 = \trace(B^2), 
\end{align*}
where $A = \sum_i a_i z_i z_i^T$ and $B = \sum_j b_j z_j z_j^T$ are chosen as described in Proposition~\ref{prop:write-as-traces}.
Since $A$ and $B$ are each symmetric, it follows that $A^2$ and $B^2$ are each symmetric and positive semidefinite. 
Then we can apply Lemma~\ref{lemma:trace-submultiplicative} to conclude that
\begin{equation*}
  \lVert f g \rVert_{\sigma / \sqrt2}^2 = \trace(A^2 B^2) \leq \trace(A^2) \trace(B^2) = \lVert f \rVert_\sigma^2 \lVert g \rVert_\sigma^2. \qedhere
\end{equation*}
\end{proof}

\subsection{Implications: kernel ridge regression}

With the help of Theorem~\ref{theorem:gaussian-product-bound}, we can develop DRO-based bounds for actual learning problems. 
In this section we develop such bounds for Gaussian kernel ridge regression, i.e. problem~\eqref{eq:gaussian-krr-wrong-sigma}. \matt{$\sigma/\sqrt2$}

For shorthand, we write $R_\distq(f) = \E_{x \sim \distq}[\ell_f(x)] = \E_{x\sim\distq}[(f(x) - h(x))^2]$ for the risk
of $f$ on a distribution $\distq$.
Generalization amounts to proving that the population risk $R_\distp(f)$ is not too different than the empirical risk $R_{\hat \distp_n}(f)$.
\begin{theorem}
\label{thm:krr-dro-bound-general}
Assume the target function $h$ satisfies $\lVert h^2 \rVert_{\sigma/\sqrt2} \leq \Lambda_{h^2}$ and $\lVert h \rVert_\sigma \leq \Lambda_h$. Then, for any $\delta > 0$, with probability $1-\delta$, the following holds for all functions $f$ satisfying $\lVert f^2 \rVert_{\sigma/\sqrt2} \leq \Lambda_{f^2}$ and $\lVert f \rVert_\sigma \leq \Lambda_f$:
\begin{align}
  R_\distp(f)
  &\leq R_{\hat \distp_n}(f) + \frac{2}{\sqrt n} \left( 1 + \sqrt{\frac{\log (1/\delta)}{2}}\right) \left( \Lambda_{f^2} + \Lambda_{h^2} + 2 \Lambda_f \Lambda_h \right).
\end{align}

\end{theorem}
\begin{proof}
We utilize the DRO Generalization Principle~\ref{principle:dro-generalization},
By Lemma~\ref{lemma:mmd-convergence} we know that with probability $1-\delta$, $\dMMD(\hat \distp_n, \distp) \leq \epsilon$ for $\epsilon = (2 + \sqrt{2 \log(1/\delta)}) / \sqrt{n}$, since $k_\sigma(x,x) \leq M = 1$.
Note the bandwidth $\sigma$ does not affect the convergence result.
As a result of Lemma~\ref{lemma:mmd-convergence}, with probability $1-\delta$:
\begin{align}
  R_\distp(f)
  &=
  \E_{x \sim \distp}[(f(x) - h(x))^2] \\
  &\overset{(a)}{\leq} \E_{x \sim \hat \distp_n}[(f(x) - h(x))^2] + \epsilon \lVert (f - h)^2 \rVert_{\sigma / \sqrt2} \\
  &\overset{(b)}{\leq} R_{\hat \distp_n}(f) + \epsilon \left( \lVert f^2 \rVert_{\sigma / \sqrt2} + \lVert h^2 \rVert_{\sigma / \sqrt2} + 2\lVert f h \rVert_{\sigma / \sqrt2}\right)  \\
  &\overset{(c)}{\leq} R_{\hat \distp_n}(f) + \epsilon \left( \Lambda_{f^2} + \Lambda_{h^2} + 2 \Lambda_f \Lambda_h \right),
\end{align}
where (a) is by Corollary~\ref{corollary:bound-dro-by-norm-penalty}, (b) is by the triangle inequality, and (c) follows from Theorem~\ref{theorem:gaussian-product-bound} and our assumptions on $f$ and $h$.
Plugging in the bound on $\epsilon$ yields the result.
\end{proof}
We placed different bounds on each of $f, h, f^2, h^2$ to emphasize the dependence on each.
Since each is bounded separately, the DRO based bound in Theorem~\ref{thm:krr-dro-bound-general}
allows finer control of the complexity of the function class than is typical.
Since, by Theorem~\ref{theorem:gaussian-product-bound}, the norms of $f^2$, $h^2$ and $f h$ are bounded by those of $f$ and $h$, we may also state Theorem~\ref{thm:krr-dro-bound-general} just with $\lVert f \rVert_\sigma$ and $\lVert h \rVert_\sigma$.
%
\begin{corollary}
Assume the target function $h$ satisfies $\lVert h \rVert_\sigma \leq \Lambda$. Then, for any $\delta > 0$, with probability $1-\delta$, the following holds for all functions $f$ satisfying $\lVert f \rVert_{\sigma} \leq \Lambda$:
\begin{align}
  R_\distp(f)
  &\leq R_{\hat \distp_n}(f) + \frac{8 \Lambda^2}{\sqrt n} \left( 1 + \sqrt{\frac{\log (1/\delta)}{2}}\right).
\end{align}
\end{corollary}
\begin{proof}
We reduce to Theorem~\ref{thm:krr-dro-bound-general}.
By Theorem~\ref{theorem:gaussian-product-bound}, we know that $\lVert f^2 \rVert_{\sigma / \sqrt2} \leq \lVert f \rVert_{\sigma}^2$, which may be bounded above by $\Lambda^2$ (and similarly for $h$).
Therefore we can take $\Lambda_{f^2} = \Lambda_f^2 = \Lambda$ and $\Lambda_{h^2} = \Lambda_h^2 = \Lambda$ in Theorem~\ref{thm:krr-dro-bound-general}.
The result follows by bounding
\begin{equation*}
  \Lambda_{f^2} + \Lambda_{h^2} + 2 \Lambda_f \Lambda_h
  \leq \Lambda^2 + \Lambda^2 + 2 \Lambda \cdot \Lambda = 4 \Lambda^2. \qedhere
\end{equation*}
\end{proof}

Generalization bounds for kernel ridge regression are of course not new;
we emphasize that the DRO viewpoint provides an intuitive approach that also grants finer control over the function complexity.
Moreover, our results take essentially the same form as the typical generalization bounds for kernel ridge regression, reproduced below:
\begin{theorem}[Specialized from~\citep{mohri2018foundations}, Theorem 10.7]
\label{theorem:krr-mohri}
Assume the target function $h$ satisfies $\lVert h \rVert_\sigma \leq \Lambda$. 
Then, for any $\delta > 0$, with probability $1-\delta$, it holds for all functions $f$ satisfying $\lVert f \rVert_\sigma \leq \Lambda$ that
\begin{align}
  R_\distp(f) &\leq R_{\hat \distp_n}(f) + \frac{8 \Lambda^2}{\sqrt n} \left( 1 + \frac12 \sqrt{\frac{\log (1/\delta)}{2}}\right).
\end{align}
\end{theorem}
Hence, our DRO-based Theorem~\ref{thm:krr-dro-bound-general} evidently recovers standard results up to a universal constant.

\subsection{Algorithmic implications}
\label{subsection:krr-alg-implications}
The generalization result in Theorem~\ref{theorem:krr-mohri} is often used to justify penalizing by the norm $\lVert f \rVert_{\sigma}$, since it is the only part of the bound (other than the risk $R_{\hat \distp_n}(f)$) that depends on $f$.
In contrast, our DRO-based generalization bound in Theorem~\ref{thm:krr-dro-bound-general} is of the form
\begin{align}
  R_\distp(f) - R_{\hat \distp_n}(f) &\leq \epsilon \left( \lVert f^2 \rVert_{\sigma / \sqrt2} + \lVert h^2 \rVert_{\sigma / \sqrt2} + 2 \lVert f \rVert_\sigma \lVert h \rVert_\sigma \right),
\end{align}
which depends on $f$ through both norms $\lVert f \rVert_\sigma$ and $\lVert f^2 \rVert_{\sigma / \sqrt2}$.
This bound motivates the use of both norms as regularizers in kernel regression, i.e. we would instead solve
\begin{align}
  \inf_{f \in \Hs_\sigma} \, \E_{(x,y) \sim \hat \distp_n}[(f(x) - y)^2] + \lambda_1 \lVert f \rVert_\sigma + \lambda_2 \lVert f^2 \rVert_{\sigma / \sqrt2}. 
\end{align}
Given data $(x_i, y_i)_{i=1}^n$, for kernel ridge regression, the Representer Theorem implies that it is sufficient to consider only $f$ of the form $f = \sum_{i=1}^n a_i k_\sigma(x_i, \cdot)$.
%
Here this is not in general possible due to the norm of $f^2$.
However, it is possible to evaluate and compute gradients of $\lVert f^2 \rVert_{\sigma / \sqrt2}^2$: 
let $K$ be the matrix with $K_{ij} = k_{\sqrt2 \sigma}(x_i, x_j)$, and let $D = \diag(a)$. 
Using Proposition~\ref{prop:write-as-traces}, we can prove 
$\lVert f^2 \rVert_{\sigma/\sqrt2}^2 = \trace( (DK)^4)$
A complete proof is given 
in the appendix.



\section{Approximation and connections to variance regularization}
\label{section:variance-reg}

In the previous section we studied bounding the MMD DRO problem~\eqref{eq:mmd-dro-prob} via Hilbert norm penalization.
Going beyond kernel methods where we search over $f \in \Hs$, it is even less clear how to evaluate the Hilbert norm $\lVert \ell_f \rVert_\Hs$.
To circumvent this issue, next we approach the DRO problem from a different angle: we directly search for the adversarial distribution $\distq$. 
Along the way, we will build connections to variance regularization~\citep{maurer2009empirical,gotoh2015robust,lam2016robust,Namkoong2017Variance}, where the empirical risk is regularized by the empirical variance of $\ell_f$: $\Var_{\hat \distp_n}(\ell_f) = \E_{x\sim\hat\distp_n}[\ell_f(x)^2] - \E_{x\sim\hat\distp_n}[\ell_f(x)]^2$.
In particular, we show in Theorem~\ref{thm:stronger-than-variance} that MMD DRO yields stronger regularization than variance.

Searching over all distributions $\distq$ in the MMD ball is intractable, so we restrict our attention to those with the same support $\{x_i\}_{i=1}^n$ as the empirical sample $\hat \distp_n$.
All such distributions $\distq$ can be written as $\distq = \sum_{i=1}^n w_i \delta_{x_i}$, where $w$ is in the $n$-dimensional simplex.
By restricting the set of candidate distributions $\distq$, we make the adversary weaker:
\begin{align}
  \begin{array}{ll}
    \sup_\distq & \E_{x\sim\distq}[\ell_f(x)] \\
    \text{s.t.} & \dMMD(\distq, \hat \distp_n) \leq \epsilon
  \end{array}
  \quad \geq \quad
  \begin{array}{ll}
    \sup_w & \sum_{i=1}^n w_i \ell_f(x_i) \\
    \text{s.t.} & \dMMD(\sum_{i=1}^n w_i \delta_{x_i}, \hat \distp_n) \leq \epsilon \\
    & \sum_{i=1}^n w_i = 1 \\
    & w_i \geq 0 \; \forall i=1,\dots,n.
    \label{eq:discrete-mmd-dro-problem}
  \end{array}
\end{align}
By restricting the support of $\distq$, it is no longer possible to guarantee out of sample performance, since it typically will have different support.
Yet, as we will see, problem~\eqref{eq:discrete-mmd-dro-problem} has nice connections.

The $\dMMD$ constraint is a quadratic penalty on $v = w - \frac1n \mathbf1$, as one may see via 
the mean embedding definition of MMD:
\begin{align}
  \dMMD\left(\sum_{i=1}^n w_i \delta_{x_i}, \hat \distp_n \right)^2
  &= \left\lVert \sum_{i=1}^n w_i k(x_i, \cdot) - \frac1n \sum_{i=1}^n k(x_i, \cdot) \right\rVert_\Hs^2 
  &= \left\lVert \sum_{i=1}^n v_i k(x_i, \cdot) \right\rVert_\Hs^2.
\end{align}
The last term is $v^T K v = (w - \frac1n \mathbf 1)^T K (w - \frac1n \mathbf 1)$, where $K$ is the kernel matrix with $K_{ij} = k(x_i,x_j)$.
If the radius $\epsilon$ of the uncertainty set is small enough, the constraints $w_i \geq 0$ are inactive, and can be ignored. 
By dropping these constraints, we can solve the adversary's problem in closed form:
\matt{originally: Dropping these constraints makes it possible to solve the adversary's problem in closed form:}
\begin{lemma}
\label{lem:discrete-opt-solution}
Let $\vec \ell$ be the vector with $i$-th element $\ell_f(x_i)$.
If $\epsilon$ is small enough that the constraints $w_i$ are not active, then the optimal value of problem~\eqref{eq:discrete-mmd-dro-problem} is given by
\begin{align}
  \E_{x\sim\hat\distp_n}[\ell_f(x)] + \epsilon \sqrt{ \vec\ell^T K^{-1} \vec\ell - \frac{(\vec\ell^T K^{-1} \mathbf 1)^2}{\mathbf 1^T K^{-1} \mathbf 1} }.
\end{align}
\end{lemma}
In other words, fitting a model to minimize the support-constrained approximation of MMD DRO is equivalent to penalizing by the nonconvex regularizer in Lemma~\ref{lem:discrete-opt-solution}.
To better understand this regularizer, consider, for instance, the case that the kernel matrix $K$ equals the identity $I$.
This will happen e.g. for a Gaussian kernel as the bandwidth $\sigma$ approaches zero. 
Then, the regularizer equals
\begin{align}
  \epsilon \sqrt{ \vec\ell^T K^{-1} \vec\ell - \frac{(\vec\ell^T K^{-1} \mathbf 1)^2}{\mathbf 1^T K^{-1} \mathbf 1} }
  = \epsilon \sqrt{ \vec\ell^T \vec\ell - \frac{(\vec\ell^T \mathbf 1)^2}{\mathbf 1^T \mathbf 1} }
  &= \epsilon \sqrt{n} \sqrt{ \Var_{\hat \distp_n}(\ell_f) }.
\end{align}

In fact, this equivalence holds a bit more generally:
\begin{lemma}
\label{lemma:var-reg-id-plus-ones}
Let $K = a I + b \mathbf 1 \mathbf 1^T$. Then,
$
  \sqrt{\vec\ell^T K^{-1} \vec\ell - \frac{(\vec\ell^T K^{-1} \mathbf 1)^2}{\mathbf 1^T K^{-1} \mathbf 1}}
  = a^{-1/2} \sqrt{n} \sqrt{ \Var_{\hat \distp_n}(\ell_f)}.
$
\end{lemma}
\matt{you can get this consequence for MMD vs $\chi^2$ without requiring $\epsilon$ small}
As a consequence, we conclude that with the right choice of kernel $k$, MMD DRO is a stronger regularizer than variance:
\begin{theorem}
\label{thm:stronger-than-variance}
There exists a kernel $k$ so that MMD DRO bounds the variance regularized problem:
\begin{align}
  \E_{x\sim\hat\distp_n}[\ell_f(x)]
  \quad \leq \quad
  \E_{x\sim\hat\distp_n}[\ell_f(x)] + \epsilon \sqrt{n} \sqrt{\Var_{\hat\distp_n}(\ell_f)}
  \quad \leq \quad
  \sup_{\distq : \dMMD(\distq, \hat \distp_n) \leq \epsilon}[\ell_f(x)].
\end{align}
\end{theorem}

\sj{Put the variance regularizer in context? Where has it been used?}


\section{Experiments}
\begin{figure}
\centering
	\includegraphics[width=3.2in]{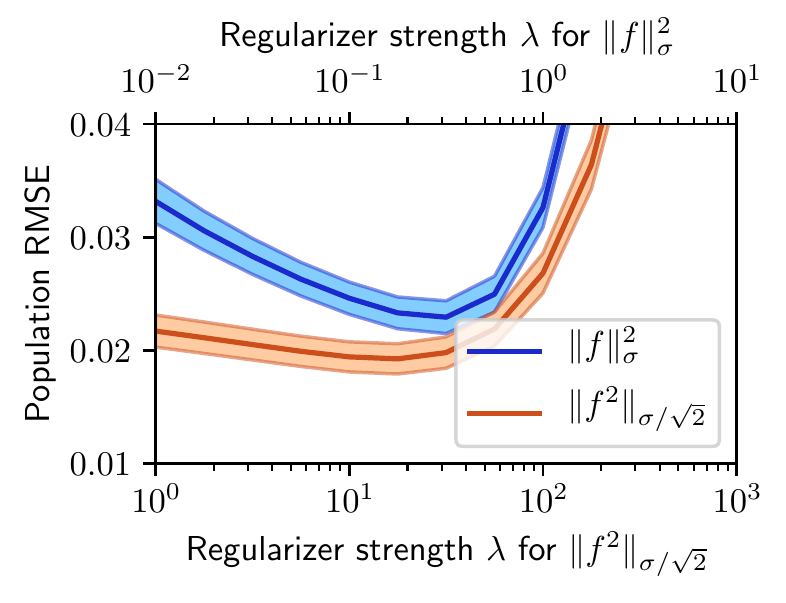}
	\includegraphics[width=3.2in]{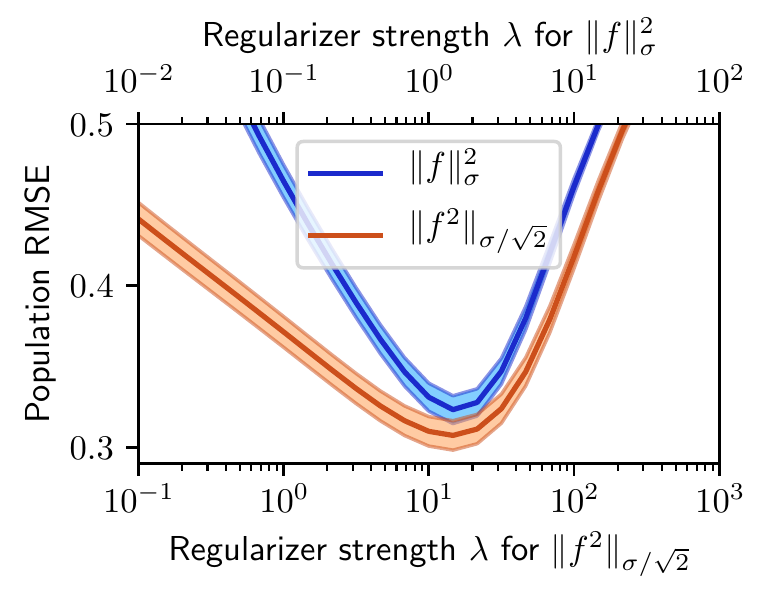}
	\vspace{-0.125in}
  \caption{Comparison of the two regularizers $\lVert f \rVert_\sigma^2$ and $\lVert f^2 \rVert_{\sigma/\sqrt2}$ in both the easy (left) and hard (right) settings, across a parameter sweep of $\lambda$. The $x$-axis is shifted to make comparison easier. }
  \label{fig:krr-new-regularizer-compare}
\end{figure}

In subsection~\ref{subsection:krr-alg-implications} we proposed an alternate regularizer for kernel ridge regression, specifically, penalizing $\lVert f^2 \rVert_{\sigma/\sqrt2}$ instead of $\lVert f \rVert_\sigma^2$.
Here we probe the new regularizer on a synthetic problem where we can precisely compute the population risk $R_\distp(f)$.
Consider the Gaussian kernel $k_\sigma$ with $\sigma=1$. 
Fix the ground truth $h = k_\sigma(1, \cdot) - k_\sigma(-1,\cdot) \in \Hs_\sigma$.
Sample $10^4$ points from a standard one dimensional Gaussian, and set this as the population $\distp$.
Then subsample $n$ points $x_i = h(x_i) + \epsilon_i$, where $\epsilon_i$ are Gaussian.
We consider both an easy regime, where $n=10^3$ and $\Var(\epsilon_i) = 10^{-2}$, and a hard regime where $n=10^2$ and $\Var(\epsilon_i) = 1$.
On the empirical data, we fit $f\in\Hs_\sigma$ by minimizing square loss plus either $\lambda \lVert f \rVert_\sigma^2$ (as is typical) or $\lambda \lVert f^2 \rVert_{\sigma/\sqrt2}$ (our proposal).
We average over $10^2$ resampling trials for the easy case and $10^3$ for the hard case, and report 95\% confidence intervals.
Figure~\ref{fig:krr-new-regularizer-compare} shows the result in each case for a parameter sweep over $\lambda$.
If $\lambda$ is tuned properly, the tighter regularizer $\lVert f^2 \rVert_{\sigma/\sqrt2}$ yields better performance in both cases.
It also appears the regularizer $\lVert f^2 \rVert_{\sigma/\sqrt2}$ is less sensitive to the choice of $\lambda$: performance decays slowly when $\lambda$ is too low.

\subsubsection*{Acknowledgements}
This work was supported by The Defense Advanced Research Projects Agency (grant number YFA17 N66001-17-1-4039). The views, opinions, and/or findings contained in this article are those of the author and should not be interpreted as representing the official views or policies, either expressed or implied, of the Defense Advanced Research Projects Agency or the Department of Defense.
We thank Cameron Musco and Joshua Robinson for helpful conversations, and Marwa El Halabi and Sebastian Claici for comments on the draft.

\bibliographystyle{plainnatnourl}
\bibliography{ref}

\begin{thebibliography}{39}
\providecommand{\natexlab}[1]{#1}
\providecommand{\url}[1]{\texttt{#1}}
\expandafter\ifx\csname urlstyle\endcsname\relax
  \providecommand{\doi}[1]{doi: #1}\else
  \providecommand{\doi}{doi: \begingroup \urlstyle{rm}\Url}\fi

\bibitem[Ben-Tal et~al.(2013)Ben-Tal, den Hertog, De~Waegenaere, Melenberg, and
  Rennen]{doi:10.1287/mnsc.1120.1641}
Aharon Ben-Tal, Dick den Hertog, Anja De~Waegenaere, Bertrand Melenberg, and
  Gijs Rennen.
\newblock Robust solutions of optimization problems affected by uncertain
  probabilities.
\newblock \emph{Management Science}, 59\penalty0 (2):\penalty0 341--357, 2013.

\bibitem[Bertsimas et~al.(2018)Bertsimas, Gupta, and Kallus]{Bertsimas2018}
Dimitris Bertsimas, Vishal Gupta, and Nathan Kallus.
\newblock Data-driven robust optimization.
\newblock \emph{Mathematical Programming}, 167\penalty0 (2):\penalty0 235--292,
  Feb 2018.

\bibitem[Bietti and Mairal(2019)]{bietti2019group}
Alberto Bietti and Julien Mairal.
\newblock Group invariance, stability to deformations, and complexity of deep
  convolutional representations.
\newblock \emph{The Journal of Machine Learning Research}, 20\penalty0
  (1):\penalty0 876--924, 2019.

\bibitem[Bietti et~al.(2019)Bietti, Mialon, Chen, and Mairal]{bietti2019kernel}
Alberto Bietti, Gr{\'e}goire Mialon, Dexiong Chen, and Julien Mairal.
\newblock A kernel perspective for regularizing deep neural networks.
\newblock In \emph{Proceedings of the 36th International Conference on Machine
  Learning}. PMLR, 2019.

\bibitem[Bińkowski et~al.(2018)Bińkowski, Sutherland, Arbel, and
  Gretton]{binkowski2018demystifying}
Mikołaj Bińkowski, Dougal~J. Sutherland, Michael Arbel, and Arthur Gretton.
\newblock Demystifying {MMD} {GAN}s.
\newblock In \emph{International Conference on Learning Representations}, 2018.

\bibitem[Blanchet et~al.(2016)Blanchet, Kang, and Murthy]{blanchet2016robust}
Jose Blanchet, Yang Kang, and Karthyek Murthy.
\newblock Robust wasserstein profile inference and applications to machine
  learning.
\newblock \emph{arXiv preprint arXiv:1610.05627}, 2016.

\bibitem[Blanchet et~al.(2017)Blanchet, Kang, Zhang, and
  Murthy]{blanchet2017data}
Jose Blanchet, Yang Kang, Fan Zhang, and Karthyek Murthy.
\newblock Data-driven optimal transport cost selection for distributionally
  robust optimization.
\newblock \emph{arXiv preprint arXiv:1705.07152}, 2017.

\bibitem[Blanchet et~al.(2018)Blanchet, Murthy, and Zhang]{blanchet2018optimal}
Jose Blanchet, Karthyek Murthy, and Fan Zhang.
\newblock Optimal transport based distributionally robust optimization:
  Structural properties and iterative schemes.
\newblock \emph{arXiv preprint arXiv:1810.02403}, 2018.

\bibitem[Chwialkowski et~al.(2016)Chwialkowski, Strathmann, and
  Gretton]{pmlr-v48-chwialkowski16}
Kacper Chwialkowski, Heiko Strathmann, and Arthur Gretton.
\newblock A kernel test of goodness of fit.
\newblock In Maria~Florina Balcan and Kilian~Q. Weinberger, editors,
  \emph{Proceedings of The 33rd International Conference on Machine Learning},
  volume~48 of \emph{Proceedings of Machine Learning Research}, pages
  2606--2615, New York, New York, USA, 20--22 Jun 2016. PMLR.

\bibitem[Delage and Ye(2010)]{doi:10.1287/opre.1090.0741}
Erick Delage and Yinyu Ye.
\newblock Distributionally robust optimization under moment uncertainty with
  application to data-driven problems.
\newblock \emph{Operations Research}, 58\penalty0 (3):\penalty0 595--612, 2010.

\bibitem[Duchi et~al.(2016)Duchi, Glynn, and Namkoong]{duchi2016statistics}
John Duchi, Peter Glynn, and Hongseok Namkoong.
\newblock Statistics of robust optimization: A generalized empirical likelihood
  approach.
\newblock \emph{arXiv preprint arXiv:1610.03425}, 2016.

\bibitem[Dziugaite et~al.(2015)Dziugaite, Roy, and
  Ghahramani]{Dziugaite:2015:TGN:3020847.3020875}
Gintare~Karolina Dziugaite, Daniel~M. Roy, and Zoubin Ghahramani.
\newblock Training generative neural networks via maximum mean discrepancy
  optimization.
\newblock In \emph{Proceedings of the Thirty-First Conference on Uncertainty in
  Artificial Intelligence}, UAI'15, pages 258--267, Arlington, Virginia, United
  States, 2015. AUAI Press.
\newblock ISBN 978-0-9966431-0-8.

\bibitem[Fournier and Guillin(2015)]{Fournier2015}
Nicolas Fournier and Arnaud Guillin.
\newblock On the rate of convergence in {Wasserstein} distance of the empirical
  measure.
\newblock \emph{Probability Theory and Related Fields}, 162\penalty0
  (3):\penalty0 707--738, Aug 2015.

\bibitem[Gao and Kleywegt(2016)]{gao2016distributionally}
Rui Gao and Anton~J Kleywegt.
\newblock Distributionally robust stochastic optimization with {Wasserstein}
  distance.
\newblock \emph{arXiv preprint arXiv:1604.02199}, 2016.

\bibitem[Gao et~al.(2017)Gao, Chen, and Kleywegt]{gao2017wasserstein}
Rui Gao, Xi~Chen, and Anton~J Kleywegt.
\newblock Wasserstein distributional robustness and regularization in
  statistical learning.
\newblock \emph{arXiv preprint arXiv:1712.06050}, 2017.

\bibitem[Goh and Sim(2010)]{doi:10.1287/opre.1090.0795}
Joel Goh and Melvyn Sim.
\newblock Distributionally robust optimization and its tractable
  approximations.
\newblock \emph{Operations Research}, 58\penalty0 (4-part-1):\penalty0
  902--917, 2010.

\bibitem[Goodfellow et~al.(2015)Goodfellow, Shlens, and
  Szegedy]{goodfellow2015explaining}
Ian~J Goodfellow, Jonathon Shlens, and Christian Szegedy.
\newblock Explaining and harnessing adversarial examples.
\newblock In \emph{International Conference on Learning Representations}, 2015.

\bibitem[Gotoh et~al.(2015)Gotoh, Kim, and Lim]{gotoh2015robust}
{Jun-ya} Gotoh, Michael Kim, and Andrew Lim.
\newblock {Robust Empirical Optimization is Almost the Same As Mean-Variance
  Optimization}.
\newblock \emph{Available at {SSRN} 2827400}, 2015.

\bibitem[Gretton et~al.(2012)Gretton, Borgwardt, Rasch, Sch\"{o}lkopf, and
  Smola]{Gretton:2012:KTT:2188385.2188410}
Arthur Gretton, Karsten~M. Borgwardt, Malte~J. Rasch, Bernhard Sch\"{o}lkopf,
  and Alexander Smola.
\newblock A kernel two-sample test.
\newblock \emph{Journal of Machine Learning Research}, 13:\penalty0 723--773,
  March 2012.

\bibitem[Hashimoto et~al.(2018)Hashimoto, Srivastava, Namkoong, and
  Liang]{pmlr-v80-hashimoto18a}
Tatsunori Hashimoto, Megha Srivastava, Hongseok Namkoong, and Percy Liang.
\newblock Fairness without demographics in repeated loss minimization.
\newblock In Jennifer Dy and Andreas Krause, editors, \emph{Proceedings of the
  35th International Conference on Machine Learning}, volume~80 of
  \emph{Proceedings of Machine Learning Research}, pages 1929--1938,
  Stockholmsmässan, Stockholm Sweden, 10--15 Jul 2018. PMLR.

\bibitem[Jitkrittum et~al.(2017)Jitkrittum, Xu, Szabo, Fukumizu, and
  Gretton]{NIPS2017_6630}
Wittawat Jitkrittum, Wenkai Xu, Zoltan Szabo, Kenji Fukumizu, and Arthur
  Gretton.
\newblock A linear-time kernel goodness-of-fit test.
\newblock In I.~Guyon, U.~V. Luxburg, S.~Bengio, H.~Wallach, R.~Fergus,
  S.~Vishwanathan, and R.~Garnett, editors, \emph{Advances in Neural
  Information Processing Systems 30}, pages 262--271. Curran Associates, Inc.,
  2017.

\bibitem[Lam(2016)]{lam2016robust}
Henry Lam.
\newblock {Robust Sensitivity Analysis for Stochastic Systems}.
\newblock \emph{Mathematics of Operations Research}, 41\penalty0 (4):\penalty0
  1248--1275, 2016.

\bibitem[Li et~al.(2015)Li, Swersky, and Zemel]{pmlr-v37-li15}
Yujia Li, Kevin Swersky, and Rich Zemel.
\newblock Generative moment matching networks.
\newblock In Francis Bach and David Blei, editors, \emph{Proceedings of the
  32nd International Conference on Machine Learning}, volume~37 of
  \emph{Proceedings of Machine Learning Research}, pages 1718--1727, Lille,
  France, 07--09 Jul 2015. PMLR.

\bibitem[Liu et~al.(2016)Liu, Lee, and Jordan]{pmlr-v48-liub16}
Qiang Liu, Jason Lee, and Michael Jordan.
\newblock A kernelized {Stein} discrepancy for goodness-of-fit tests.
\newblock In Maria~Florina Balcan and Kilian~Q. Weinberger, editors,
  \emph{Proceedings of The 33rd International Conference on Machine Learning},
  volume~48 of \emph{Proceedings of Machine Learning Research}, pages 276--284,
  New York, New York, USA, 20--22 Jun 2016. PMLR.

\bibitem[Madry et~al.(2018)Madry, Makelov, Schmidt, Tsipras, and
  Vladu]{madry2018towards}
Aleksander Madry, Aleksandar Makelov, Ludwig Schmidt, Dimitris Tsipras, and
  Adrian Vladu.
\newblock Towards deep learning models resistant to adversarial attacks.
\newblock In \emph{International Conference on Learning Representations}, 2018.

\bibitem[Maurer and Pontil(2009)]{maurer2009empirical}
Andreas Maurer and Massimiliano Pontil.
\newblock Empirical {Bernstein} bounds and sample variance penalization.
\newblock In \emph{Conference on Learning Theory}, 2009.

\bibitem[Mohajerin~Esfahani and Kuhn(2018)]{MohajerinEsfahani2018}
Peyman Mohajerin~Esfahani and Daniel Kuhn.
\newblock Data-driven distributionally robust optimization using the
  {Wasserstein} metric: performance guarantees and tractable reformulations.
\newblock \emph{Mathematical Programming}, 171\penalty0 (1):\penalty0 115--166,
  Sep 2018.

\bibitem[Mohri et~al.(2018)Mohri, Rostamizadeh, and
  Talwalkar]{mohri2018foundations}
Mehryar Mohri, Afshin Rostamizadeh, and Ameet Talwalkar.
\newblock \emph{Foundations of machine learning}.
\newblock MIT press, 2018.

\bibitem[Muandet et~al.(2017)Muandet, Fukumizu, Sriperumbudur, and
  Schölkopf]{MAL-060}
Krikamol Muandet, Kenji Fukumizu, Bharath Sriperumbudur, and Bernhard
  Schölkopf.
\newblock Kernel mean embedding of distributions: A review and beyond.
\newblock \emph{Foundations and Trends® in Machine Learning}, 10\penalty0
  (1-2):\penalty0 1--141, 2017.

\bibitem[Namkoong and Duchi(2017)]{Namkoong2017Variance}
Hongseok Namkoong and John~C. Duchi.
\newblock {Variance-based Regularization with Convex Objectives}.
\newblock In \emph{Advances in Neural Information Processing Systems 30}, pages
  2975--2984, 2017.

\bibitem[Scarf(1958)]{scarf1958min}
Herbert Scarf.
\newblock A min-max solution of an inventory problem.
\newblock \emph{Studies in the mathematical theory of inventory and
  production}, 1958.

\bibitem[Shafieezadeh~Abadeh et~al.(2015)Shafieezadeh~Abadeh,
  Mohajerin~Esfahani, and Kuhn]{NIPS2015_5745}
Soroosh Shafieezadeh~Abadeh, Peyman~Mohajerin Mohajerin~Esfahani, and Daniel
  Kuhn.
\newblock Distributionally robust logistic regression.
\newblock In C.~Cortes, N.~D. Lawrence, D.~D. Lee, M.~Sugiyama, and R.~Garnett,
  editors, \emph{Advances in Neural Information Processing Systems 28}, pages
  1576--1584. Curran Associates, Inc., 2015.

\bibitem[Sinha et~al.(2018)Sinha, Namkoong, and Duchi]{sinha2018certifying}
Aman Sinha, Hongseok Namkoong, and John Duchi.
\newblock Certifying some distributional robustness with principled adversarial
  training.
\newblock In \emph{International Conference on Learning Representations}, 2018.

\bibitem[Staib and Jegelka(2017)]{staib2017distributionally}
Matthew Staib and Stefanie Jegelka.
\newblock Distributionally robust deep learning as a generalization of
  adversarial training.
\newblock In \emph{NIPS Machine Learning and Computer Security Workshop}, 2017.

\bibitem[Staib et~al.(2019)Staib, Wilder, and
  Jegelka]{staib2019distributionally}
Matthew Staib, Bryan Wilder, and Stefanie Jegelka.
\newblock Distributionally robust submodular maximization.
\newblock In Kamalika Chaudhuri and Masashi Sugiyama, editors,
  \emph{Proceedings of the Twenty-Second International Conference on Artificial
  Intelligence and Statistics}, volume~89 of \emph{Proceedings of Machine
  Learning Research}, pages 506--516. PMLR, 16--18 Apr 2019.

\bibitem[Sutherland et~al.(2017)Sutherland, Tung, Strathmann, De, Ramdas,
  Smola, and Gretton]{sutherland2017generative}
Dougal~J Sutherland, Hsiao-Yu Tung, Heiko Strathmann, Soumyajit De, Aaditya
  Ramdas, Alex Smola, and Arthur Gretton.
\newblock Generative models and model criticism via optimized maximum mean
  discrepancy.
\newblock In \emph{International Conference on Learning Representations}, 2017.

\bibitem[Szegedy et~al.(2014)Szegedy, Zaremba, Sutskever, Bruna, Erhan,
  Goodfellow, and Fergus]{szegedy2014intriguing}
Christian Szegedy, Wojciech Zaremba, Ilya Sutskever, Joan Bruna, Dumitru Erhan,
  Ian Goodfellow, and Rob Fergus.
\newblock Intriguing properties of neural networks.
\newblock In \emph{International Conference on Learning Representations}, 2014.

\bibitem[Villani(2008)]{villani2008optimal}
Cédric Villani.
\newblock \emph{Optimal Transport: Old and New (Grundlehren der mathematischen
  Wissenschaften)}.
\newblock Springer, 2008.
\newblock ISBN 9788793102132.

\bibitem[Xu et~al.(2009)Xu, Caramanis, and Mannor]{xu2009robustness}
Huan Xu, Constantine Caramanis, and Shie Mannor.
\newblock Robustness and regularization of support vector machines.
\newblock \emph{Journal of Machine Learning Research}, 10\penalty0
  (Jul):\penalty0 1485--1510, 2009.

\end{thebibliography}

\appendix


\section{Proofs of main structural results}
\label{appendix:structural-results-proofs}

\begin{proof}[Proof of Theorem~\ref{thm:sup-of-mean-embeddings-norm}]
We will use weak duality to derive a candidate solution, and then use that solution to show strong duality.
First, note that
\begin{align}
  \sup_{\mu_\distq \in \Hs : \lVert \mu_\distq - \mu_\distp \rVert_\Hs \leq \epsilon} \langle f, \mu_\distq \rangle_\Hs
  &= \sup_{\mu_\distq \in \Hs} \inf_{\lambda \geq 0} \left\{ \langle f, \mu_\distq \rangle_\Hs - \lambda ( \lVert \mu_\distq - \mu_\distp \rVert_\Hs^2 - \epsilon^2) \right\} \\
  &\leq \inf_{\lambda \geq 0} \sup_{\mu_\distq \in \Hs} \left\{ \langle f, \mu_\distq \rangle_\Hs - \lambda ( \lVert \mu_\distq - \mu_\distp \rVert_\Hs^2 - \epsilon^2) \right\} \\
  &= \inf_{\lambda \geq 0} \left\{ \lambda \epsilon^2 + \sup_{\mu_\distq \in \Hs} \left\{ \langle f, \mu_\distq \rangle_\Hs - \lambda \lVert \mu_\distq - \mu_\distp \rVert_\Hs^2 \right\} \right\}.
\end{align}
We first focus on the innermost objective, which may be rewritten:
\begin{align}
  \langle f, \mu_\distq \rangle_\Hs - \lambda \lVert \mu_\distq - \mu_\distp \rVert_\Hs^2 
  &= \langle f, \mu_\distp \rangle_\Hs + \langle f, \mu_\distq - \mu_\distp \rangle_\Hs - \lambda \lVert \mu_\distq - \mu_\distp \rVert_\Hs^2 \\
  &= \langle f, \mu_\distp \rangle_\Hs - \lambda \left[ \lVert \mu_\distq - \mu_\distp \rVert_\Hs^2 - 2 \left\langle \frac{1}{2\lambda} f, \mu_\distq - \mu_\distp \right\rangle_\Hs \right] \\
  &= \langle f, \mu_\distp \rangle_\Hs - \lambda \left[ \left\lVert \mu_\distq - \mu_\distp - \frac{1}{2\lambda} f\right\rVert_\Hs^2 + \left\lVert \frac{1}{2\lambda} f \right\rVert_\Hs^2 \right],
\end{align}
where the final inequality is by completing the square.
Only one term depends on $\mu_\distq$, namely $-\lambda \lVert \mu_\distq - \mu_\distp - \frac{1}{2\lambda}f \rVert_\Hs^2$;
since norms are nonnegative, this term can never exceed zero, and zero is achieved by $\mu_\distq^* = \mu_\distp + \frac{1}{2\lambda} f \in \Hs$, yielding inner objective value
\begin{align}
  \langle f, \mu_\distp \rangle_\Hs - \lambda \left\lVert \frac{1}{2\lambda} f \right\rVert_\Hs^2 
  = \langle f, \mu_\distp \rangle_\Hs - \frac{1}{4\lambda} \lVert f \rVert_\Hs^2.
\end{align}
Plugging this in for the inner problem, and then solving for the optimal dual variable $\lambda^*$, we derive the upper bound:
\begin{align}
  \sup_{\mu_\distq \in \Hs : \lVert \mu_\distq - \mu_\distp \rVert_\Hs \leq \epsilon} \langle f, \mu_\distq \rangle_\Hs
  &\leq \inf_{\lambda \geq 0} \left\{ \lambda \epsilon^2 + \langle f, \mu_\distp \rangle_\Hs + \frac{1}{4\lambda} \lVert f \rVert^2_\Hs \right\} \\
  &= \langle f, \mu_\distp \rangle_\Hs + \epsilon \lVert f \rVert_\Hs.
\end{align}
The optimal dual variable $\lambda^* = \frac{1}{2\epsilon} \lVert f \rVert_\Hs$ is that which balances the two terms. Plugging this in, we find that $\mu_\distq^* = \mu_\distp + \frac{\epsilon}{\lVert f \rVert_\Hs} f$.

In order to prove equality, it remains to show strong duality holds. We will achieve this by lower bounding the original objective. Specifically, the supremum over all $\mu_\distq$ can be lower bounded by plugging in our particular $\mu_\distq^*$:
\begin{align}
  \sup_{\mu_\distq \in \Hs : \lVert \mu_\distq - \mu_\distp \rVert_\Hs \leq \epsilon} \langle f, \mu_\distq \rangle_\Hs
  &= \sup_{\mu_\distq \in \Hs} \inf_{\lambda \geq 0} \left\{ \langle f, \mu_\distq \rangle_\Hs - \lambda ( \lVert \mu_\distq - \mu_\distp \rVert_\Hs^2 - \epsilon^2) \right\} \\  
  &\geq \inf_{\lambda \geq 0} \left\{ \langle f, \mu_\distq^* \rangle_\Hs - \lambda ( \lVert \mu_\distq^* - \mu_\distp \rVert_\Hs^2 - \epsilon^2) \right\} \\  
  &= \inf_{\lambda \geq 0} \left\{ \left\langle f, \mu_\distp + \frac{\epsilon}{\lVert f \rVert_\Hs} f \right\rangle_\Hs - \lambda \left( \left\lVert \frac{\epsilon}{\lVert f \rVert_\Hs} f \right\rVert_\Hs^2 - \epsilon^2 \right) \right\} \\  
  &= \inf_{\lambda \geq 0} \left\{ \left\langle f, \mu_\distp + \frac{\epsilon}{\lVert f \rVert_\Hs} f \right\rangle_\Hs - \lambda \left( \epsilon^2 - \epsilon^2 \right) \right\} \\  
  &= \left\langle f, \mu_\distp + \frac{\epsilon}{\lVert f \rVert_\Hs} f \right\rangle_\Hs
  = \langle f, \mu_\distp \rangle_\Hs + \epsilon \lVert f \rVert_\Hs.
\end{align}
Since the same bound appears on both sides, we have equality.
\end{proof}


\section{Gaussian kernel bounds}
\label{appendix:gaussian-kernel-bounds}

We first reproduce Proposition~\ref{prop:write-as-traces} for convenience:
\begin{proposition}
Let $f, g \in \Hs_\sigma$ have the expansions $f = \sum_i a_i k_\sigma(x_i, \cdot)$ and $g = \sum_j b_j k_\sigma(x_j, \cdot)$.
For shorthand denote by $z_i = \phi_{\sqrt2 \sigma}(x_i)$ the (possibly infinite) feature expansion of $x_i$ in $\Hs_{\sqrt2\sigma}$.
Then 
\begin{align*}
  \lVert f g \rVert_{\sigma / \sqrt2}^2 = \trace(A^2 B^2), 
  \quad
  \lVert f \rVert_\sigma^2 = \trace(A^2), 
  \quad \text{and} \quad \lVert g \rVert_\sigma^2 = \trace(B^2), 
\end{align*}
where $A = \sum_i a_i z_i z_i^T$ and $B = \sum_j a_j z_j z_j^T$.
\end{proposition}

In order to prove Proposition~\ref{prop:write-as-traces}, we will need a utility lemma that helps translate between $\Hs_\sigma$ and $\Hs_{\sigma / \sqrt2}$:
\begin{lemma}
\label{lemma:inner-product-quad-product}
  Let $\langle \cdot , \cdot \rangle_{\sigma / \sqrt2}$ be the inner product in the RKHS $\Hs_{\sigma / \sqrt2}$. Let $\langle \cdot, \cdot \rangle_{\sigma'}$ refer to the inner product in $H_{\sigma'}$. Then,
  \begin{align}
    \label{eq:kernel-prod-inner-prod}
    \langle k_\sigma(x, \cdot) k_\sigma(y, \cdot), k_\sigma(a, \cdot) k_\sigma(b, \cdot) \rangle_{\sigma / \sqrt2}
  \end{align}
  can be simplified as
  \begin{align}
    k_{\sigma \sqrt2}(x, a) k_{\sigma \sqrt2}(x, b) k_{\sigma \sqrt2}(y, a) k_{\sigma \sqrt2}(y, b).
  \end{align}
\end{lemma}

In order to make the proof cleaner, we first derive a couple of identities involving norms and sums.
\begin{lemma}
\label{lemma:norms-sums-identity-1}
For vectors $x, y, z$, the following identity holds:
\begin{align}
  \lVert x - z \rVert^2 + \lVert y - z \rVert^2
  = \frac12 \lVert x - y \rVert^2 + 2 \left\lVert z - \frac{x+y}{2} \right\rVert^2
\end{align}
\end{lemma}
\begin{proof}
Simply expand:
\begin{align}
  \lVert x - z \rVert^2 + \lVert y - z \rVert^2
  &= \lVert x \rVert^2 + \lVert y \rVert^2 + 2 \lVert z \rVert^2 - 2 z^T (x+y) \\
  &= \lVert x \rVert^2 + \lVert y \rVert^2 + 2 \left( \lVert z \rVert^2 - 2 z^T \left(\frac{x+y}{2}\right) \right) \\
  &= \lVert x \rVert^2 + \lVert y \rVert^2 + 2 \left( \left\lVert z - \frac{x+y}{2} \right\rVert^2 - \left\lVert \frac{x+y}{2} \right\rVert^2 \right) \\
  &= \lVert x \rVert^2 + \lVert y \rVert^2 + 2 \left\lVert z - \frac{x+y}{2} \right\rVert^2 - \frac12 \lVert x+y \rVert^2 \\
  &= \frac12 \lVert x - y \rVert^2 + 2 \left\lVert z - \frac{x+y}{2} \right\rVert^2. \qedhere
\end{align}
\end{proof}

\begin{lemma}
\label{lemma:norms-sums-identity-2}
Let $x, y, a, b$ be arbitrary vectors, and define $S$ and $T$ by:
\begin{align*}
  S &:= \lVert x - y \rVert^2 + \lVert a - b \rVert^2 + \lVert (x+y) - (a+b) \rVert^2 \\
  T &:= \lVert x - a \rVert^2 + \lVert x - b \rVert^2 + \lVert y - a \rVert^2 + \lVert y - b \rVert^2.
\end{align*}
Then $S = T$.
\end{lemma}
\begin{proof}
Start by expanding the third term of $S$:
\begin{align}
  &\lVert x - y \rVert^2 + \lVert a - b \rVert^2 + \lVert (x+y) - (a+b) \rVert^2 \\
  = \; &\lVert x - y \rVert^2 + \lVert a - b \rVert^2 + \lVert (x-a) + (y-b) \rVert^2 \\
  = \; &\lVert x - y \rVert^2 + \lVert a - b \rVert^2 + 2 (x-a)^T (y-b) + \lVert x - a \rVert^2 + \lVert y - b \rVert^2.
  \label{eq:expanded-quad-norm-sum}
\end{align}
The first three terms of equation~\eqref{eq:expanded-quad-norm-sum} can be expanded as
\begin{align}
  &\lVert x - y \rVert^2 + \lVert a - b \rVert^2 + 2 (x-a)^T (y-b) \\
  = \; &\lVert x \rVert^2 + \lVert y \rVert^2 - 2x^T y + \lVert a \rVert^2 + \lVert b \rVert^2 - 2a^T b + 2 (x-a)^T (y-b) \\
  = \; &\lVert x \rVert^2 + \lVert y \rVert^2 - 2x^T y + \lVert a \rVert^2 + \lVert b \rVert^2 - 2a^T b + 2 x^T y - 2x^T b - 2a^T y + 2a^T b \\
  = \; &\lVert x \rVert^2 + \lVert y \rVert^2 + \lVert a \rVert^2 + \lVert b \rVert^2  - 2x^T b - 2a^T y \\
  = \; &\lVert x - b \rVert^2 + \lVert y - a \rVert^2.
\end{align}
Replacing the first three terms in equation~\eqref{eq:expanded-quad-norm-sum} by $\lVert x - b \rVert^2 + \lVert y -a \rVert^2$ yields $T$, i.e. $S = T$.
\end{proof}

We are now equipped to prove Lemma~\ref{lemma:inner-product-quad-product}:
\begin{proof}[Proof of Lemma~\ref{lemma:inner-product-quad-product}]
First, write
\begin{align}
  k_\sigma(x,z) k_\sigma(y,z) &= \exp\left(-\frac{1}{2\sigma^2} \left( \lVert x - z \rVert^2 + \lVert y - z \rVert^2 \right)\right) \\
  &= \exp\left(-\frac{1}{2\sigma^2} \left( \frac12 \lVert x - y \rVert^2 + 2 \left\lVert z - \frac{x+y}{2} \right\rVert^2 \right) \right) \\
  &= \exp\left(-\frac{1}{4\sigma^2} \lVert x - y \rVert^2 \right) \exp\left( -\frac{1}{\sigma^2}\left\lVert z - \frac{x+y}{2} \right\rVert^2 \right) \\
  &= k_{\sigma \sqrt2}(x,y) k_{\sigma / \sqrt2}\left(z, \frac{x+y}{2}\right),
\end{align}
where in the second line we used Lemma~\ref{lemma:norms-sums-identity-1}.
Note that the first term does not depend on $z$. Now, applying this identity to Equation~\eqref{eq:kernel-prod-inner-prod}, we find:
\begin{align}
  &\langle k_\sigma(x, \cdot) k_\sigma(y, \cdot), k_\sigma(a, \cdot) k_\sigma(b, \cdot) \rangle_{\sigma / \sqrt2} \\
  = \; &k_{\sigma \sqrt2}(x,y) k_{\sigma \sqrt2}(a,b) \left\langle k_{\sigma/\sqrt2}\left(\frac{x+y}{2}, \cdot \right),  k_{\sigma/\sqrt2}\left(\frac{a+b}{2}, \cdot \right) \right\rangle_{\sigma / \sqrt2} \\
  = \; &k_{\sigma \sqrt2}(x,y) k_{\sigma \sqrt2}(a,b) k_{\sigma/\sqrt2}\left( \frac{x+y}{2}, \frac{a+b}{2} \right) \\
  = \; &k_{\sigma \sqrt2}(x,y) k_{\sigma \sqrt2}(a,b) k_{\sigma \sqrt2}\left( x+y, a+b \right).
  \label{eq:k-triple-product}
\end{align}
To simplify this expression, 
notice that it takes the form $\exp(-S / (4\sigma^2))$, where 
\begin{align}
  S &= \lVert x - y \rVert^2 + \lVert a - b \rVert^2 + \lVert (x + y) - (a + b) \rVert^2.
\end{align}
By Lemma~\ref{lemma:norms-sums-identity-2}, $S$ is equal to
\begin{align}
  S = \lVert x - a \rVert^2 + \lVert x - b \rVert^2 + \lVert y - a \rVert^2 + \lVert y - b \rVert^2,
\end{align}
which means equation~\eqref{eq:k-triple-product} can be rewritten as
\begin{align*}
  \exp\left(-\frac{S}{4\sigma^2} \right)
  &= \exp\left(-\frac{\lVert x - a \rVert^2}{4\sigma^2} \right) \exp\left(-\frac{\lVert x - b \rVert^2}{4\sigma^2} \right) \exp\left(-\frac{\lVert y - a \rVert^2}{4\sigma^2} \right) \exp\left(-\frac{\lVert y - b \rVert^2}{4\sigma^2} \right) \\
  &= k_{\sigma \sqrt2}(x, a) k_{\sigma \sqrt2}(x, b) k_{\sigma \sqrt2}(y, a) k_{\sigma \sqrt2}(y, b).
\end{align*}
\end{proof}

With Lemma~\ref{lemma:inner-product-quad-product} available, it is possible to prove Proposition~\ref{prop:write-as-traces}:
\begin{proof}[Proof of Proposition~\ref{prop:write-as-traces}]
Define the vectors $z_i$ as described, so that $z_i^T z_j = k_{\sqrt2 \sigma}(x_i, x_j)$. For convenience, also write $K_{ij} = k_{\sqrt2 \sigma}(x_i, x_j)$, and observe that $K_{ij}^2 = k_\sigma(x_i, x_j)$. It follows that
\begin{align}
	\lVert f \rVert_\sigma^2 
	= \sum_i \sum_{j} a_i a_{j} k_\sigma(x_i, x_j)
	= \sum_i \sum_{j} a_i a_{j} K_{ij}^2 
	&= \sum_i \sum_{j} a_i a_{j} z_i^T z_j z_j^T z_i
\end{align}
Rearranging the inner terms, we find
\begin{align}
	\lVert f \rVert_\sigma^2
	= \sum_i a_i z_i^T \left( \sum_{j} a_{j} z_j z_j^T \right) z_i 
	= \sum_i a_i z_i^T A z_i 
	= \trace\left( \sum_i a_i z_i z_i^T A \right)
	= \trace(A^2),
\end{align}
where we have used the definition of $A$, the fact that the trace of a scalar is simply that scalar, and the cyclic property of the trace. The proof that $\lVert g \rVert_\sigma^2 = \trace(B^2)$ is identical, so we omit it.

The derivation of the trace form of $\lVert f g \rVert_{\sigma / \sqrt2}^2$ is more complicated. 
Expanding out $f g$, we see that
\begin{align}
  (fg)(x) = \sum_{i,j} a_i b_j k_\sigma(x_i, x) k_\sigma(x_j, x).
\end{align}
Therefore the norm $\lVert f g \rVert_{\sigma / \sqrt2}^2$, which is simply $\langle f g, f g \rangle_{\sigma / \sqrt2}$, is equal to:
\begin{align}
  \langle f g, f g \rangle_{\sigma / \sqrt2} 
  &= \left\langle \sum_{i,j} a_i b_j k_\sigma(x_i, x) k_\sigma(x_j, x), \sum_{i',j'} a_{i'} b_{j'} k_\sigma(x_{i'}, x) k_\sigma(x_{j'}, x) \right\rangle_{\sigma / \sqrt2} \\
  &= \sum_{i,j,i',j'} a_i a_{i'} b_j b_{j'} \left\langle k_\sigma(x_i, x) k_\sigma(x_j, x), k_\sigma(x_{i'}, x) k_\sigma(x_{j'}, x) \right\rangle_{\sigma / \sqrt2} \\
  &= \sum_{i,j,i',j'} a_i a_{i'} b_j b_{j'} k_{\sigma\sqrt2}(x_i,x_{i'}) k_{\sigma\sqrt2}(x_i,x_{j'}) k_{\sigma\sqrt2}(x_j,x_{i'}) k_{\sigma\sqrt2}(x_j,x_{j'}) \\
  &= \sum_{i,j,i',j'} a_i a_{i'} b_j b_{j'} K_{ii'} K_{ij'} K_{ji'} K_{jj'},
\end{align}
where in the second to last step we have used Lemma~\ref{lemma:inner-product-quad-product}.
Before continuing, 
observe the identity
\begin{align}
  \sum_\ell a_\ell K_{i \ell} K_{j \ell} 
  = \sum_\ell a_\ell z_i^T z_\ell z_\ell^T z_j
  = z_i^T \left( \sum_\ell a_\ell z_\ell z_\ell^T \right) z_j
  = z_i^T A z_j
\end{align}
Similarly, $\sum_\ell b_\ell K_{i \ell}  K_{j \ell} = z_i^T B z_j$.
Leveraging these identities, we continue:
\begin{align}
  \sum_{i,j,i',j'} a_i a_{i'} b_j b_{j'} K_{ii'} K_{ij'} K_{ji'} K_{jj'}
  &= \sum_{i,i',j} a_i a_{i'} b_{j} K_{ii'} K_{ji'} \sum_{j'} b_{j'} K_{ij'} K_{jj'} \\
  &= \sum_{i,i',j} a_i a_{i'} b_{j} K_{ii'} K_{ji'} ( z_{i}^T B z_{j} ) \\
  &= \sum_{i,j} a_i  b_{j} \left( \sum_{i'} a_{i'} K_{ii'} K_{ji'} \right) (z_{i}^T B z_{j} ) \\
  &= \sum_{i,j} a_i  b_{j} (z_j^T A z_i)( z_{i}^T B z_{j}).
\end{align}
At this point we leverage the cyclic property of the trace, so the above expression equals:
\begin{equation*}
  \trace\left( \sum_{i,j} a_i  b_{j} A z_i z_{i}^T B z_{j} z_j^T \right)
  = \trace\left(  A \left( \sum_i a_i z_i z_{i}^T \right) B \left( \sum_j b_j z_{j} z_j^T \right) \right)
  = \trace(A^2 B^2). \qedhere
\end{equation*}

\end{proof}

\subsection{Trace inequality}

\begin{proof}[Proof of Lemma~\ref{lemma:trace-submultiplicative}]
Consider the trace inner product $\langle X, Y \rangle = \trace(X^T Y) = \trace(X Y)$, where the final equality is because $X$ is symmetric. By the Cauchy-Schwarz inequality, we have $\trace(X Y) \leq \sqrt{\trace(X^2) \trace(Y^2)}$.
Let $\{\lambda_i\}_{i=1}^n$ be the eigenvalues of $X$. Then,
\begin{align}
  \trace(X^2) 
  = \sum_{i=1}^n \lambda_i^2
  \leq \sum_{i=1}^n \lambda_i^2 + 2\sum_{i=1}^n \sum_{j=i+1}^n \lambda_i \lambda_j
  = \left( \sum_{i=1}^n \lambda_i \right)^2
  = \trace(X)^2,
\end{align}
where the inequality holds because $\lambda_i$ are all nonnegative. The same holds for any positive semidefinite matrix, in particular, $Y$. Combining these two inequalities, we have
\begin{align}
  \trace(X Y) \leq \sqrt{\trace(X^2) \trace(Y^2)} \leq \sqrt{\trace(X)^2 \trace(Y)^2} = \trace(X) \trace(Y).
\end{align}
\end{proof}

\subsection{Extensions of Proposition~\ref{prop:write-as-traces}}
There are many useful corollaries and extensions of Proposition~\ref{prop:write-as-traces}.
First, we give a result that makes it tractable to actually compute $\lVert f g \rVert_{\sigma / \sqrt2}$:
\begin{corollary}
\label{corollary:compute-f-square-norm}
Suppose $f = \sum_{i=1}^n a_i k_\sigma(x_i, \cdot)$ and $g = \sum_{i=1}^n b_i k_\sigma(x_i, \cdot)$ have the same finite expansion, but with potentially different coefficients. Form the kernel matrix $K$ with $K_{ij} = k_{\sqrt2 \sigma}(x_i, x_j)$, where we have replaced the bandwidth $\sigma$ with $\sqrt2\sigma$. 
Write $D_a = \diag(a)$ and similarly for $D_b$.
Then, 
\begin{align}
  \lVert f g \rVert_{\sigma / \sqrt2}^2 = \trace( (D_a K)^2 (D_b K)^2 ).
\end{align}
\end{corollary}
\begin{proof}
Pick vectors $z_i$ so that $z_i^T z_j = K_{ij}$,
and let $Z$ be the matrix with $i$-th column $z_i$.
Note that $A = \sum_{i=1}^n a_i z_i z_i^T = Z D_a Z^T$, and similarly for $B$.
Then we may write
\begin{align}
  \lVert f g \rVert_{\sigma / \sqrt2}^2 
  &\stackrel{(a)}{=} \trace( A^2 B^2 ) \\
  &= \trace( (Z D_a Z^T) (Z D_a Z^T) (Z D_b Z^T) (Z D_b Z^T) )\\
  &\stackrel{(b)}{=} \trace( D_a Z^T Z D_a Z^T Z D_b Z^T Z D_b Z^T Z )\\
  &\stackrel{(c)}{=} \trace( D_a K D_a K D_b K D_b K ) \\
  &= \trace( (D_a K)^2 (D_b K)^2 ),
\end{align}
where (a) is by Proposition~\ref{prop:write-as-traces}, (b) is by the cyclic property of the trace, and (c) follows since $Z^T Z = K$ by definition of $z_i$.
\end{proof}

\section{Proofs for Section~\ref{section:variance-reg}}
\begin{proof}[Proof of Lemma~\ref{lem:discrete-opt-solution}]
For notational convenience, we just write $\ell$ instead of $\vec \ell$.
First, notice that problem~\eqref{eq:discrete-mmd-dro-problem}, once the $w_i \geq 0$ constraint is dropped, can be written
\begin{align}
  \begin{array}{ll}
    \sup_w & \ell^T w \\
    \mathrm{s.t.} & \left( w - \frac1n \mathbf 1 \right)^T K \left( w - \frac1n \mathbf 1 \right) \leq \epsilon^2 \\
    & \mathbf 1^T w = 1
    \label{prob:discrete-mmd-dro}
  \end{array} 
\end{align}

Write $v = w - \frac1n \mathbf 1$. Then the value of problem~\eqref{prob:discrete-mmd-dro} is equal to 
\begin{align}
  \frac1n \mathbf 1^T \ell + 
  \begin{array}{ll}
    \sup_v & \ell^T v \\
    \mathrm{s.t.} & v^T K v \leq \epsilon^2 \\
    & \mathbf 1^T v = 0
  \end{array} 
\end{align}
and we can focus on this slightly simpler problem. This problem can be in turn rewritten as:
\begin{align}
  \sup_v \inf_{\eta \geq 0, \lambda} \left\{ \ell^T v - \eta(v^T K v - \epsilon^2) - \lambda \mathbf 1^T v \right\}.
\end{align}
By Slater's condition, strong duality holds, so the optimal value is equal to:
\begin{align}
  &\inf_{\eta \geq 0, \lambda} \left\{ \eta \epsilon^2 + \sup_v \left\{\ell^T v - \eta v^T K v - \lambda \mathbf 1^T v \right\} \right\} \\
  = \; &\inf_{\eta \geq 0, \lambda} \left\{ \eta \epsilon^2 + \sup_v \left\{- \eta v^T K v + (\ell - \lambda \mathbf 1)^T v \right\} \right\}.
\end{align}
The inner problem is a concave quadratic maximization problem. In general, if $A$ is symmetric, $-x^T A x + b^T x$ is maximized when $x = \frac12 A^{-1} b$, and the resulting objective value is $\frac14 b^T A^{-1} b$. Applying this to the problem at hand, we find that the optimal $v^*$ satisfies:
\begin{align}
  v^* = \frac{1}{2\eta} K^{-1} (\ell - \lambda \mathbf1),
\end{align}
and the corresponding objective value of the inner problem is
\begin{align}
  \frac{1}{4\eta} (\ell - \lambda \mathbf1)^T K^{-1} (\ell - \lambda \mathbf1).
\end{align}
The overall problem is therefore
\begin{align}
  \inf_{\eta \geq 0, \lambda} \left\{ \eta \epsilon^2 + \frac{1}{4\eta} (\ell - \lambda \mathbf1)^T K^{-1} (\ell - \lambda \mathbf1) \right\}.
\end{align}
The objective is a convex quadratic in $\lambda$, and it is simple to check that $\lambda^* = (\mathbf 1^T K^{-1} \ell) / (\mathbf 1^T K^{-1} \mathbf 1)$. Then, both remaining terms are positive, so it is optimal to balance them. This leads to
\begin{align}
  \eta^* \epsilon^2 &= \frac{1}{4 \eta^*} (\ell - \lambda^* \mathbf1)^T K^{-1} (\ell - \lambda^* \mathbf1) \\
  \implies \frac{1}{2\eta^*} &= \frac{\epsilon}{\sqrt{(\ell - \lambda^* \mathbf1)^T K^{-1} (\ell - \lambda^* \mathbf1)}},
\end{align}
and the overall optimal value is
\begin{align}
  &2 \cdot \frac{1}{4\eta^*} (\ell - \lambda^* \mathbf1)^T K^{-1} (\ell - \lambda^* \mathbf1) \\
  = \; &\epsilon \sqrt{(\ell - \lambda^* \mathbf1)^T K^{-1} (\ell - \lambda^* \mathbf1)}.
\end{align}
The term inside the square root is equal to
\begin{align}
  (\ell - \lambda^* \mathbf1)^T K^{-1} (\ell - \lambda^* \mathbf1)
  &= \ell^T K^{-1} \ell - 2 \lambda^* \mathbf 1^T K^{-1} \ell + (\lambda^*)^2 \mathbf 1^T K^{-1} \mathbf 1 \\
  &= \ell^T K^{-1} \ell - \frac{(\mathbf 1^T K^{-1} \ell)^2}{\mathbf 1^T K^{-1} \mathbf 1},
\end{align}
from which we can simply compute the overall objective of the original problem.
\end{proof}

\begin{proof}[Proof of Lemma~\ref{lemma:var-reg-id-plus-ones}]
One can prove via the matrix inversion lemma that
\begin{align}
  K^{-1} = (a I + b \mathbf1 \mathbf1^T)^{-1}
  &= a^{-1} \left[ I - \frac{b}{a + bn} \mathbf1 \mathbf1^T \right].
\end{align}
As a consequence,
\begin{align}
  a \ell^T K^{-1} \ell
  &= \lVert \ell \rVert^2 - \frac{b}{a+bn} (\mathbf 1^T \ell)^2 \\
  a \ell^T K^{-1} \mathbf1
  &= \mathbf1^T \ell - \frac{b}{a+bn} (\mathbf 1^T \ell) (\mathbf 1^T \mathbf 1)
  = \frac{a}{a+bn} \cdot \mathbf1^T \ell \\
  a \mathbf1^T K^{-1} \mathbf1
  &= \mathbf1^T \mathbf1 - \frac{b}{a+bn} (\mathbf 1^T \mathbf1)^2 
  = \frac{a}{a+bn} \cdot n.
\end{align}
It follows that
\begin{align}
  \frac{(a \ell^T K^{-1} \mathbf1)^2}{a \mathbf 1^T K^{-1} \mathbf1}
  = a \cdot \frac{\left( \frac{a}{a+bn} \mathbf1^T \ell \right)^2}{\frac{a}{a+bn} \cdot n}
  = a \cdot \frac{( \mathbf1^T \ell )^2}{n} \cdot \frac{a}{a+bn}
  = a \cdot ( \mathbf1^T \ell )^2 \cdot \left( \frac1n - \frac{b}{a+bn} \right)
\end{align}
and therefore
\begin{align}
  a \cdot \left[ \ell^T K^{-1} \ell - \frac{(\ell^T K^{-1} \mathbf1)^2}{\mathbf 1^T K^{-1} \mathbf1} \right]
  &= a \cdot \left[ \lVert \ell \rVert^2 - \frac{b}{a+bn} (\mathbf 1^T \ell)^2 - ( \mathbf1^T \ell )^2 \cdot \left( \frac1n - \frac{b}{a+bn} \right) \right] \\
  &= a \cdot \left[ \lVert \ell \rVert^2 - \frac{( \mathbf1^T \ell )^2}{n} \right]
  = a \cdot \Var_{\hat \distp_n}(\ell),
\end{align}
from which the conclusion follows.
\end{proof}

\end{document}